\documentclass[conference,10pt]{IEEEtran}
\usepackage{graphicx,color}
\usepackage{amsmath, amsthm, amsfonts, amssymb, amsbsy,nccmath,bbm}
\usepackage{algorithm}
\usepackage{algpseudocode}
\usepackage{enumerate}
\usepackage{lipsum}
\newtheorem{lemma}{Lemma}

\usepackage[sort,compress]{cite}
\usepackage{epsfig}
\usepackage{epstopdf}
\usepackage{mathabx}
\usepackage{mathtools}
\usepackage{dsfont}
\usepackage{epstopdf}

\usepackage[inline]{enumitem}   
\makeatletter
\newcommand{\inlineitem}[1][]{%
\ifnum\enit@type=\tw@
    {\descriptionlabel{#1}}
  \hspace{\labelsep}%
\else
  \ifnum\enit@type=\z@
       \refstepcounter{\@listctr}\fi
    \quad\@itemlabel\hspace{\labelsep}%
\fi}
\makeatother
\newcommand{\beq}{\begin{equation}}
\newcommand{\eeq}{\end{equation}}

\def\adots{\mathinner{\mskip0mu\raise0pt\vbox{\kern7pt\hbox{.}}\mskip3mu
          \raise4pt\hbox{.}\mskip3mu\raise8pt\hbox{.}\mskip0mu}}

\newcommand{\bmc}{{\mathbf c}}

\newcommand{\xh}{\widehat{x}}

\newcommand{\tr}{\mbox{tr}}

\usepackage{bm}

\newcommand{\bmtheta}{{\bm \theta}}

\newcommand{\bme}{{\bm e}}
\newcommand{\bmx}{{\bm x}}
\newcommand{\bmy}{{\bm y}}

\newcommand{\bmH}{{\bm H}}
\newcommand{\bmQ}{{\bm Q}}

\newcommand{\bmxh}{\widehat{\bmx}}

\newcommand{\bmA}{{\bm A}}

\newcommand{\bmC}{{\bf C}}

\newcommand{\bmI}{{\bf {I}}}

\newcommand{\bmm}{{\bf {m}}}

\newcommand{\bma}{{\bm a}}

\newcommand{\bmU}{\bm U}

\newcommand{\thetah}{\widehat{\theta}}

\makeatletter
\newcommand\fs@spaceruled{\def\@fs@cfont{\bfseries}\let\@fs@capt\floatc@ruled
  \def\@fs@pre{\vspace{0.5\baselineskip}\hrule height.8pt depth0pt \kern2pt}%
  \def\@fs@post{\kern1pt\hrule\relax}%
  \def\@fs@mid{\kern2pt\hrule\kern2pt}%
  \let\@fs@iftopcapt\iftrue}
\makeatother

\newcommand{\bit}{\begin{itemize}}
\newcommand{\eit}{\end{itemize}}

\newcommand{\bmu}{\boldsymbol{\mu}}
\newcommand{\bmz}{\mathbf{z}}

\renewcommand{\bmA}{{\mathbf A}}

\newcommand{\bmB}{{\mathbf B}}

\renewcommand{\bmU}{{\mathbf U}}

\newcommand{\bmb}{{\mathbf b}}

\usepackage{lipsum}

\newcommand{\bmXi}{{\boldsymbol \Xi}}

\DeclareMathOperator{\E}{\mathbb{E}}
\newcommand{\var}{\mbox{var}}
\renewcommand{\bmxh}{\widehat{\bmx}}

\renewcommand{\bmu}{\bm u}

\newcommand{\zh}{\widehat{z}}

\newcommand{\bmthetab}{{\overline{\bmtheta}}}

\newcommand{\KLD}{{\mathrm{KLD}}}

\newcommand{\cbN}{{\overline{\mathcal{N}}}}
\newcommand{\cbM}{{\overline{\mathcal{M}}}}
\newcommand{\bmnu}{{\bm{\nu}}}

\newcommand{\bmxc}{{\widecheck{\bmx}}}
\newcommand{\nb}{{\overline{n}}}

\usepackage{subfiles}

\usepackage[acronym]{glossaries}

\newacronym{kld}{KLD}{Kullback–Leibler divergence}
\newacronym{snr}{SNR}{signal-to-noise ratio}
\newacronym{ap}{AP}{access point}

\newtheorem*{remark}{Remark}

\newtheorem{proposition}{Proposition}

\begin{document}
\linespread{0.82}

\title{Avoiding Non-Integrable Beliefs in Expectation Propagation}
\author{%
  \IEEEauthorblockN{Zilu Zhao, Jichao Chen, Dirk Slock}
  \IEEEauthorblockA{
			\small
			Communication Systems Department, EURECOM, France \\	
			zilu.zhao@eurecom.fr, jichao.chen@eurecom.fr, dirk.slock@eurecom.fr
			\vspace{-3mm}		}
	}

\maketitle

\begin{abstract}
Expectation Propagation (EP) is a widely used iterative message-passing algorithm that decomposes a global inference problem into multiple local ones. It approximates marginal distributions as ``beliefs'' using intermediate functions called ``messages''. It has been shown that the stationary points of EP are the same as corresponding constrained Bethe Free Energy (BFE) optimization problem. Therefore, EP is an iterative method of optimizing the constrained BFE. However, the iterative method may fall out of the feasible set of the BFE optimization problem, i.e., the beliefs are not integrable.
In most literature, the authors use various methods to keep all the messages integrable. In most Bayesian estimation problems, limiting the messages to be integrable shrinks the actual feasible set. Furthermore, in extreme cases where the factors are not integrable, making the message itself integrable is not enough to have integrable beliefs. In this paper, two EP frameworks are proposed to ensure that EP has integrable beliefs. Both of the methods allows non-integrable messages. We then investigate the signal recovery problem in Generalized Linear Model (GLM) using our proposed methods.

\vspace{-2mm}
\end{abstract}

\section{Introduction}
\label{Intro}

Signal recovery is a fundamental problem in signal processing, with a wide range of applications.
In the Bayesian framework, however, canonical inference methods such as MMSE and MAP suffer from exponential computational complexity as the problem dimension increases.

Graphical-model-based iterative methods have proven effective by exploiting structural properties of the underlying models \cite{wainwright2008graphical}.
Expectation Propagation (EP), introduced in \cite{minka2005divergence}, transforms a global inference problem into multiple local inference problems.
EP can be interpreted as an iterative procedure for finding the stationary point of the constrained Bethe Free Energy (BFE) \cite{heskes2005approximate, zhang2021unifying}.
Within this framework, EP approximates marginal distributions (beliefs) using functions known as messages, which correspond one-to-one with the Lagrange multipliers in the constrained BFE formulation. The iterative message passing is interpreted as the alternating optimization of the corresponding constrained BFE.

In EP and constrained BFE minimization, beliefs are required to be integrable. Whereas, the messages are allowed to have any form. However, the alternating optimization or EP cannot foresee whether a message update will lead to non-integrable belief.

In most literature \cite{ngo2020multi, fischer2020vamp, rangan2019vector}, the authors restrict the algorithm to integrable messages to avoid non-integrable beliefs. However, such augmentations are suboptimal, since optimal solution may contain non-integrable messages \cite{zhao2025expectations, zhao2026approximating}.

\subsection{Main Contributions}
In \cite{zhao2025expectations}, we investigate the non-integrable beliefs and message in linear model. However, the proposed methods only work for simple factor graph and cannot be extended to more general ones such as Generalized Linear Model (GLM) or bilinear estimation problems. Furthermore, the analytical continuation technique was only used to restrict the messages to be integrable. 

In this paper, we extend the ideas in \cite{zhao2025expectations} and propose two frameworks for EP to avoid non-integrable belief. Both of the methods requires sequential update. GLM is then investigated using our proposed EP frameworks.

\section{Analytic Continuation Expectation Propagation (ACEP)}
Consider a factored distribution
\beq
    p(\bmtheta)\propto \prod_{\alpha} f_{\alpha}(\bmtheta_{\alpha}).
\eeq

We adopt a sequential update scheme to infer the marginal distribution of $p(\bmtheta)$ as $b^{\thetah}_{n}(\theta_n)\simeq \int p(\bmtheta) \prod_{n'\neq n} d\theta_{n'}$. The belief at factor $f_\alpha$ can be computed as
\beq
    b^{f}_{\alpha}(\bmtheta_{\alpha})\propto f_{\alpha}(\bmtheta_{\alpha})\prod_{n\in N(\alpha)}\mu^{\theta\to f}_{\alpha, n}(\theta_n)
\eeq
The marginal belief is 
\beq
    b^{\thetah;f}_{\alpha, n}(\theta_n)=\int b^{f}_{\alpha}(\bmtheta_{\alpha}) \prod_{n'\in N(\alpha)/\{n\}}d\theta_{n'}.
\eeq
It is then projected to Gaussian distribution by KLD
\beq
    \min_{\mu^{f\to \theta}_{\alpha, n}(\theta_n)\in \mathcal F_n}\KLD[b^{\thetah;f}_{\alpha, n}(\theta_n)\| \frac{1}{Z}\mu^{f\to \theta}_{\alpha, n}(\theta_n)\mu^{\theta\to f}_{\alpha, n}(\theta_n)],
\eeq
where $\mathcal{F}_n$ ensures that $b^{\thetah}_n(\theta_n)$ is integrable.

The belief at the variable node is
\beq
    b^{\thetah}_{n}(\theta_n)\propto \prod_{\alpha\in N(n)}\mu^{f\to \theta}_{\alpha, n}(\theta_n)
\eeq
It is then projected to Gaussian by KLD
\beq
    \min_{\mu^{\theta\to f}_{\alpha, n}(\theta_n)\in \mathcal{F_{\alpha}}} \KLD[b^{\thetah}_{n}(\theta_n)\|\frac{1}{Z}\mu^{f\to \theta}_{\alpha, n}(\theta_n)\mu^{\theta\to f}_{\alpha, n}(\theta_n)],
\eeq
where the domain $\mathcal F_{\alpha}$ ensures that the belief $b^{f}_{\alpha}(\bmtheta_{\alpha})$ is integrable.

\subsection{Constrained KLD Optimization in Gaussian Projection}
\label{sect:constrainedKLDopt}
Gaussian projection is commonly used in EP. 
When using Gaussian projection, a common problem is ``negative variance'', which may lead to non-integrable belief.
 non-integrable messages are acceptable, while non-integrable belief prevents the algorithm from proceeding.
We define unnormalized Gaussian
\beq
    \cbN(\bmtheta|\bmm, \bmC)=\exp[-(\bmtheta-\bmm)^H\bmC^{-1}(\bmtheta-\bmm)].
\eeq
Sometimes, it is more convenient to use natural parameters, we define unnormalized Gaussian with natural parameter to be
\beq
    \cbM(\bmtheta|\bmnu, \bmXi)=\exp[-(\bmtheta-\bmXi^{-1}\bmnu)^{H}\bmXi(\bmtheta-\bmXi^{-1}\bmnu)].
\eeq

Therefore, we have
\beq
    [(\bmnu, \bmXi)=(\bmC^{-1}\bmm, \bmC^{-1})]\Leftrightarrow[\cbN(\bmtheta|\bmm, \bmC)=\cbM(\bmtheta|\bmnu, \bmXi)]
\eeq

We consider the constrained projection step
\beq
\begin{split}
    \min_{\mu_{p}(\theta)} \quad &\KLD[b(\theta)\|\frac{1}{Z}\mu_{r}(\theta)\mu_{p}(\theta)] \\
    \text{s.t.} \quad & \mu_{p}\in \mathcal{F},
\end{split}
\eeq
where $\mathcal{F}$ is a subset of unnormalized Gaussian. We often only have constraint on the variance (or precision) and the above problem becomes
\beq
\begin{split}
    \min_{\nu_{p}, \xi_{p}} \quad &\KLD[b(\theta)\|\frac{1}{Z(\nu_p, \xi_{p})}\mu_{r}(\theta)\cbM(\theta|\nu_p, \xi_p)] \\
    \text{s.t.} \quad & \xi_p>\gamma,
\end{split}
\label{eq:spawc2613}
\eeq

We define 
\beq
    m_{p}=\xi_{p}\nu_{p}
\eeq
and thus,
\beq
\cbM(\theta|\nu_p, \xi_p)=\exp[-(\theta-m_p)^{T}\xi_{p}(\theta-m_{p})].
\eeq
Denote 
\beq
\begin{split}
    m_{\thetah}=\E_{b}[\theta]\\
    \tau_{\thetah}=\var_{b}[\theta].
\end{split}
\eeq
For simplicity, denote 
\beq
    D=2\cdot \KLD[b(\theta)\|\frac{1}{Z(\nu_p, \xi_{p})}\mu_{r}(\theta)\cbM(\theta|\nu_p, \xi_p)]
\eeq
to remove the constant scale.
Expand the KLD in \eqref{eq:spawc2613}, 
\beq
\begin{split}
    D=-\log[\det(\xi_{p}+\xi_{r})]+\log(2\pi)\\
    +\tr[(\xi_{p}+\xi_{r})\tau_{\thetah}]
    +[m_{\thetah}-(\xi_{p}+\xi_{r})^{-1}(\nu_{p}+\nu_{r})]^T\\
    \cdot(\xi_{p}+\xi_{r})[m_{\thetah}-(\xi_{p}+\xi_{r})^{-1}(\nu_{p}+\nu_{r})]
\end{split}
\eeq
We set the derivative w.r.t. $\nu_{p}$ to zero, and we have
\beq
    \begin{split}
        (\xi_{p}+\xi_{r})^{-1}(\nu_{p}+\nu_{r})=m_{\thetah}\\
        \Rightarrow \nu_p=(\xi_{p}+\xi_{r})m_{\thetah}-\nu_r
    \end{split}
    \label{eq:spawc2618}
\eeq
Substitute the result into the KLD, and we have
\beq
\begin{split}
    D_{|\nu_p=(\xi_{p}+\xi_{r})m_{\thetah}-\nu_r}=-\log[\det(\xi_{p}+\xi_{r})]+\log(2\pi)\\
    +\tr[(\xi_{p}+\xi_{r})\tau_{\thetah}]
\end{split}
\label{eq:spawc2619}
\eeq
We can verify that \eqref{eq:spawc2619} is convex even in multi-variate ($\theta$ is vector, $\xi_{p}$, $\xi_{r}$, $\tau_{\thetah}$ are matrices) cases. In multi-variate EP, the constraint in \eqref{eq:spawc2613} becomes an intersection of positive-definite constraints, which forms an intersection of cones. Thus, the feasible set is also convex. Therefore, the constrained optimization of \eqref{eq:spawc2619} is a convex optimization problem. We will only consider univariate case here, since multivariate does not change the convexity.

Set the derivative of \eqref{eq:spawc2619} w.r.t. $\xi_{p}$, and take the constraint into consideration
\beq
    \xi_{p}=\max\{\frac{1}{\tau_{\thetah}}-\xi_{r}, \gamma\}.
    \label{eq:spawc2620}
\eeq
After that, $\nu_{p}$ can be updated by substituting \eqref{eq:spawc2620} into \eqref{eq:spawc2618}. Since updating one message will only affect the belief at either a single factor or variable node, this modification won't increase the complexity order and it ensures that all the beliefs are integrable during each step.

\section{EP with Checking (EPC)}
Here we do not check if the beliefs are integrable. Instead, we skip the update of the messages from the node if the belief at the node is not integrable.

The belief at factor $f_\alpha$ can be computed as
\beq
    b^{f}_{\alpha}(\bmtheta_{\alpha})\propto f_{\alpha}(\bmtheta_{\alpha})\prod_{n\in N(\alpha)}\mu^{\theta\to f}_{\alpha, n}(\theta_n)
\eeq
If this belief is not integrable, we just skip the message update from this node.
The marginal belief is 
\beq
    b^{\thetah;f}_{\alpha, n}(\theta_n)=\int b^{f}_{\alpha}(\bmtheta_{\alpha}) \prod_{n'\in N(\alpha)/\{n\}}d\theta_{n'}.
\eeq
It is then projected to Gaussian distribution by KLD
\beq
    \min_{\mu^{f\to \theta}_{\alpha, n}(\theta_n)\in \cbM}\KLD[b^{\thetah;f}_{\alpha, n}(\theta_n)\| \frac{1}{Z}\mu^{f\to \theta}_{\alpha, n}(\theta_n)\mu^{\theta\to f}_{\alpha, n}(\theta_n)],
    \label{eq:spawc2623}
\eeq
where ${\cbM}$ denote the unnormalized Gaussian.

The belief at the variable node is
\beq
    b^{\thetah}_{n}(\theta_n)\propto \prod_{\alpha\in N(n)}\mu^{f\to \theta}_{\alpha, n}(\theta_n)
\eeq

The message from the variable is updated as
\beq
    \mu^{\theta\to f}_{\alpha, n}(\theta_n)\propto \frac{b^{\theta}_{n}(\theta_n)}{\mu^{f\to \theta}_{\alpha, n}(\theta_n)}.
\eeq

We can verify that the belief at variable node is always integrable. In fact, we can rewrite the belief at variable node as
\beq
    b^{\thetah}_{n}(\theta_n)\propto \mu^{\theta\to f}_{\alpha, n}(\theta_n) \mu^{f\to \theta}_{\alpha, n}(\theta_n),
    \label{eq:spawc2626}
\eeq
which coincide with the second parameter in \eqref{eq:spawc2623}. Since we use unnormalized Gaussian family in the KLD projection, after updating the message $\mu^{f\to \theta}_{\alpha, n}(\theta_n)$ based on \eqref{eq:spawc2623}, the belief computed by \eqref{eq:spawc2626} is a Gaussian with the same mean and variance as $b^{\theta;f}_{\alpha, n}(\theta_n)$.

\section{System Model for Generalized Linear Model}
We consider a generalized linear model (GLM), the unknown input $\bmx$ follows 
\beq
\bmx\sim p(\bmx)=\prod_{k=1}^K p(x_k).
\eeq
The observed output follows
\beq
    \bmy\sim p(\bmy|\bmA^T\bmx)=\prod_{n=1:N}p(y_n|\bma_{n}^T\bmx).
\eeq

We introduce an auxiliary variable $\bmz=\bmA^T\bmx$. The joint PDF can be written as
\beq
\begin{split}
    p(\bmy, \bmz, \bmx)=p(\bmy|\bmz)p(\bmz|\bmx)p(\bmx)\\
    =\prod_{n}f_{y_{n}}(z_n) \cdot f_{\bmz}(\bmz, \bmx)
    \cdot \prod_{k}f_{x_{k}}(x_{k})
\end{split}
\eeq

where for simplicity, we denote
\begin{align}
    f_{y_{n}}(z_n)=p(y_{n}|z_{n})\\
    f_{\bmz}(\bmz, \bmx)=\delta(\bmz-\bmA\bmx)\\
    f_{x_{k}}(x_{k})=p(x_k)
\end{align}

\section{Derivation for GLM}

\subsection{$f_{y_n}\to z_n$}
The belief at $f_{y_{n}}$ is
\beq
    b^{f_{y}}_{n}(z_{n})\propto f_{y_{n}}(z_{n})\mu^{z\to f_{y}}_{n}(z_{n}).
\eeq
In ACEP, the above belief guarantees to be integrable. Meanwhile in EPC, the update of messages from the above belief is skipped if the above belief is not integrable. We assume it is integrable. The mean and variance are computed as
\beq
\begin{split}
    m^{\zh;f_{y}}_{n}=\E_{b^{f_y}_{n}}[z_n]\\
    \tau^{\zh;f_{y}}_{n}=\var_{b^{f_{y}}_{n}}[z_n].
\end{split}
\eeq
Next, we compute the threshold precision $\gamma^{f_{y}\to z}_{n}$. We only need to compute the threshold in ACEP when $\xi^{z\to f_{y}}_{n}$ is updated with the threshold value. Since the belief at the variable $z_n$ is required to be integrable, we have
\beq
    \gamma^{f_y\to z}_{n}+\xi^{f_{\bmz}\to z}_{n}>0\Rightarrow \gamma^{f_{y}\to z}_{n}>-\xi^{f_{\bmz}\to z}_{n}.
\eeq
Therefore, we set 
\beq
    \gamma^{f_{y}\to z}_{n}=-\xi^{f_{\bmz}\to z}_{n}+\epsilon,
\eeq
where $\epsilon$ is a small positive number. Depending on which method is used, we have different ways of updating the message precision:
\begin{itemize}
    \item If ACEP is used, we have
    \beq
        \xi^{f_{y}\to z}_{n}=\max\{\frac{1}{\tau^{\zh;f_{y}}_{n}}-\xi^{z\to f_{y}}_{n}, \gamma^{f_{y}\to z}_{n}\}.
    \eeq
    \item If EPC is used, we have
    \beq
        \xi^{f_{y}\to z}_{n}=\frac{1}{\tau^{\zh;f_{y}}_{n}}-\xi^{z\to f_{y}}_{n}.
    \eeq
\end{itemize}

After that, we update the other parameter as
\beq
    \nu^{f_{y}\to z}_{n}=(\xi^{f_{y}\to z}_{n}+\xi^{z\to f_{y}}_{n})m^{\zh;f_{y}}_{n}-\nu^{z\to f_{y}}_{n}.
\eeq

\subsection{$f_{x_k}\to x_k$}
The belief at $f_{x_k}$ is
\beq
    b^{f_{x}}_{k}(x_{k})\propto f_{x_{k}}(x_{k})\mu^{x\to f_{x}}_{k}(x_{k}).
\eeq
Its mean and variance are computed as
\beq
\begin{split}
    m^{\xh;f_{x}}_{k}=\E_{b^{f_{x}}_{k}}[x_k]\\
    \tau^{\xh;f_{x}}_{k}=\var_{b^{f_{x}}_{k}}[x_k].
\end{split}
\eeq
The threshold $\gamma^{f_{x}\to x}_{k}$ can be obtained by investigating $b^{\xh}_{k}(x_{k})$,
\beq
\begin{split}
    \gamma^{f_{x}\to x}_k+\xi^{f_{\bmz}\to x}_k>0\\
    \Rightarrow \gamma^{f_{x}\to x}_k=-\xi^{f_{\bmz}\to x}_k+\epsilon
\end{split}
\eeq
After that, we can obtain the update for message $\mu^{f_{x}\to x}_{k}$.

\subsection{$f_{\bmz}\to x_k$}
We look at the belief at $f_{\bmz}$
\beq
    b^{f_{\bmz}}(\bmz, \bmx)\propto \delta(\bmz-\bmA\bmx)\prod_{n}\mu^{z\to f_{\bmz}}_{n}(z_n)\prod_{k}\mu^{x\to f_{\bmz}}_{k}(x_k).
\eeq
Marginalize it over $\bmz$ results to
\beq
    \int b^{f_{\bmz}}(\bmz, \bmx) d\bmz\propto \cbN(\bmx|\bmm^{\bmxh;f_{\bmz}}, \bmC^{\bmxh;f_{\bmz}}),
\eeq
where
\beq
    \begin{split}
        \bmC^{\bmxh;f_{\bmz}}=[\bmA^{T}(\bmC^{\bmz\to f_{\bmz}})^{-1}\bmA+(\bmC^{\bmx\to f_{\bmz}})^{-1}]^{-1}\\
        \bmm^{\bmxh;f_{\bmz}}=\bmC^{\bmxh;f_{\bmz}}[\bmA^{T}(\bmC^{\bmz\to f_{\bmz}})^{-1}\bmm^{\bmz\to f_{\bmz}}+(\bmC^{\bmx\to f_{\bmz}})^{-1}\bmm^{\bmx\to f_{\bmz}}]
    \end{split}
\eeq
Denote $\bme_k$ as a unit vector whose $k$-th entry is one, and the marginalized belief is
\beq
    b^{\xh;f_{\bmz}}_{k}(x_k)=\int \int b^{f_{\bmz}}(\bmz, \bmx) d\bmz \prod_{k'\neq k}dx_{k'}=\cbN(x_k|m^{\xh;f_{\bmz}}_{k}, \tau^{\xh;f_{\bmz}}_{k}),
\eeq
where 
\beq
\begin{split}
    \tau^{\xh;f_{\bmz}}_{k}=\bme_{k}^T\bmC^{\bmxh;f_{\bmz}}\bme_{k}\\
    m^{\xh;f_{\bmz}}_{k}=\bme_{k}^T\bmm^{\bmxh;f_{\bmz}}.
\end{split}
\label{eq:spawc2648p}
\eeq
To derive the threshold $\gamma^{f_{\bmz}\to x}_{k}$, we look at the belief at $x_k$:
\beq
\begin{split}
    \gamma^{f_{\bmz}\to x}_{k}+\xi^{f_{x}\to x}_{k}>0\\
    \Rightarrow\gamma^{f_{\bmz}\to x}_{k}=-\xi^{f_{x}\to x}_{k}+\epsilon.
\end{split}
\eeq
From this point, we can compute the feed back message by either of the methods:
\beq
    \mu^{f_{\bmz}\to x}_{k}(x_k)=\cbM(x_k|\nu^{f_{\bmz}\to x}_{k}, \xi^{f_{\bmz}\to x}_{k}).
\eeq

\subsection{$f_{\bmz}\to z_n$}
We look at the belief at $f_{\bmz}$.
\beq
    b^{f_{\bmz}}(\bmz, \bmx)\propto \delta(\bmz-\bmA\bmx)\prod_{n}\mu^{z\to f_{\bmz}}_{n}(z_n)\prod_{k}\mu^{x\to f_{\bmz}}_{k}(x_k)
\eeq
We investigate the marginal belief of $z_{n}$ first,
\beq
    b^{\zh;f_{\bmz}}_{n}(z_{n})\propto \int \int b^{f_{\bmz}}(\bmz, \bmx) \prod_{n'\neq n}d z_{n'} d\bmx.
\eeq
Use Lemma \ref{lemma:spawc261} to integrate $b^{f_{\bmz}}(z_n, \bmx)d\bmx$ on the second line of \eqref{eq:spawc2643}:
\beq
\begin{split}
    b^{\zh;f_{\bmz}}_{n}(z_{n})\propto \cbN(\bmm^{\bmz\to f_{\bmz}}_{\nb}|\bmA_{\nb}\bmm^{\bmx\to f_{\bmz}}, \bmC^{\bmz\to f_{\bmz}}_{\nb}+\bmA_{\nb}\bmC^{\bmx\to f_{\bmz}}\bmA_{\nb}^T)\\
    \cdot\int \delta(z_{n}-\bma_{n}^T\bmx) \cbN(z_{n}|m^{z\to f_{\bmz}}_{n}, \tau^{z\to f_{\bmz}}_{n})\cbN(\bmx|\bmm^{\bmxc}_{n}, \bmC^{\bmxc}_{n})d\bmx\\
    \propto \cbN(z_n|\bma_{n}^T\bmm^{\bmxh;f_{\bmz}}, \bma_{n}^T\bmC^{\bmxh;f_{\bmz}}\bma_{n})
\end{split}
\label{eq:spawc2643}
\eeq
where
\beq
\begin{split}
    \bmC^{\bmxc}_{n}=[\bmA_{\nb}^T(\bmC^{\bmz\to f_{\bmz}}_{\nb})^{-1}\bmA_{\nb}+(\bmC^{\bmx\to f_{\bmz}})^{-1}]^{-1}.\\
    \bmm^{\bmxc}_{n}=\bmC^{\bmxc}_{n}[\bmA_{\nb}^T(\bmC^{\bmz\to f_{\bmz}}_{\nb})^{-1}\bmm^{\bmz\to f_{\bmz}}_{\nb}+(\bmC^{\bmx\to f_{\bmz}})^{-1}\bmm^{\bmx\to f_{\bmz}}]
\end{split}
\eeq

We also have
\beq
    \bma_{n}^T\bmC^{\bmxh;f_{\bmz}}\bma_{n}=\bme_{n}^T[(\bmA\bmC^{\bmx\to f_{\bmz}}\bmA^T)^{-1}+(\bmC^{\bmz\to f_{\bmz}})^{-1}]^{-1}\bme_{n}
\eeq
if the inversion exists.

An interesting relation is $\forall \bmnu$: 
\beq
    \bma_{n}^T\bmC^{\bmxc}_{n}\bmnu=\tau^{z\to f_{\bmz}}_{n}(\bma_{n}^T\bmC^{\bmxh;f_{\bmz}}\bma_{n}-\tau^{z\to  f_{\bmz}}_{n})^{-1}\bma_n^T\bmC^{\bmxh;f_{\bmz}}\bmnu
\eeq

In fact, $\cbN(\bmx|\bmm^{\bmxc}_{n}, \bmC^{\bmxc}_{n})$ is an approximation of $p(\bmx|\forall n'\neq n:y_{n'})$.

The mean and variance of the  marginal belief of $z_{n}$ at $f_{\bmz}$ is 
\beq
\begin{split}
    m^{\zh; f_{\bmz}}_{n}=\E_{b^{\zh; f\bmz}_{n}}[z_{n}]=\bma_{n}^T\bmm^{\bmxh;f_{\bmz}}\\
    \tau^{\zh;f_{\bmz}}_{n}=\var_{b^{\zh; f\bmz}_{n}}[z_{n}]=\bma_{n}^T\bmC^{\bmxh;f_{\bmz}}\bma_{n}
\end{split}
\label{eq:spawc2647}
\eeq
To derive the threshold, we observe the destination node of message $\mu^{f_{\bmz}\to z}_{n}$. In order to make the belief at $b^{\zh}_{n}(z_n)$ integrable, we have
\beq
\begin{split}
    \xi^{f_{y}\to z}_{n}+\gamma^{f_{\bmz}\to z}_{n}>0\\
    \gamma^{f_{\bmz}\to z}_{n}=-\xi^{f_{y}\to z}_{n}+\epsilon
\end{split}
\label{eq:spawc2648}
\eeq
Based on \eqref{eq:spawc2647} and \eqref{eq:spawc2648}, we can update the message $\mu^{f_{\bmz}\to z}_{n}(z_n)$ based on the EP algorithm discussed before.

\subsection{$z\to f_{y}$}
The belief at $z_n$ is
\beq
    b^{\zh}_{n}(z_n)=\mathcal{N}(z_n|m^{\zh}_{n}, \tau^{\zh}_{n})\propto \mu^{f_{\bmz}\to z}_{n}(z_n)\mu^{f_{y}\to z}_{n}(z_n),
\eeq
where
\beq
    \begin{split}
        \tau^{\zh}_{n}=(\xi^{f_{\bmz}\to z}_{n}+\xi^{f_{y}\to z}_{n})^{-1}\\
        m^{\zh}_{n}=\tau^{\zh}_{n}(\nu^{f_{\bmz}\to z}_{n}+\nu^{f_{y}\to z}_{n}).
    \end{split}
    \label{eq:spawc2660}
\eeq
We then determine the threshold precision $\gamma^{z\to f_{y}}_{n}$ based on the factor $f_{y_n}(z_n)$. For example, if $f_{y_n}(z_n)$ is proportional to a Gaussian Mixture Model, the sum of $\gamma^{z\to f_{y}}_{n}$ and the precision of the smallest GMM elements should be greater than zero.

\subsection{$z\to f_{\bmz}$}
We need to determine the threshold precision $\gamma^{z\to f_{\bmz}}_{n}$. The update of one message correspond to a rank-one update of the covariance matrix $\bmC^{\bmxh;f_{\bmz}}$
\beq
    (\bmC^{\bmxh;f_{\bmz}})^{-1}+\bma_{n}(\gamma^{z\to f_{\bmz}}_{n}-\xi^{z\to f_{\bmz}}_{n})\bma_{n}\succ 0,
\eeq
where $\xi^{z\to f_{z}}_{n}$ is the message variance before updating and $\bmC^{\bmxh;f_{\bmz}}$ is the covariance matrix before updating.

From Lemma \ref{lemma:spawc262}, the positive definite condition is equivalent to
\beq
    \gamma^{z\to f_{\bmz}}_{n}>\xi^{z\to f_{\bmz}}_{n}-(\bma_{n}^T\bmC^{\bmxh;f_{\bmz}}\bma_{n})^{-1}=\xi^{z\to f_{\bmz}}_{n}-\xi^{\zh;f_{\bmz}}_{n}.
    \label{eq:spawc2662p}
\eeq

\begin{proposition}
\label{prop:spawc261}
    In ACEP, with proper initialization (e.g. all positive message precision), if the update of $\mu^{f_{\bmz}\to z}_{n}$, $\mu^{f_{y}\to z}_{n}$, $\mu^{z\to f_{y}}_{n}$ and $\mu^{z\to f_{\bmz}}_{n}$ follows the following order:
    \beq
        \mu^{f_{\bmz}\to z}_{n} \rightarrow \mu^{z\to f_{y}}_{n} \rightarrow \mu^{f_{y}\to z}_{n}\rightarrow \mu^{z\to f_{\bmz}}_{n},
        \label{eq:spawc2664}
    \eeq
    the update of $\mu^{f_{\bmz}\to z}_{n}$ and $\mu^{z\to f_{\bmz}}_{n}$ does not need to evaluate the threshold.
\end{proposition}
\begin{proof}
    We can prove this by mathematical induction. We use superscript $t$ to denote the iteration number. Assume the proposition hold till $t-1$ iteration of \eqref{eq:spawc2664}. 
    By the induction hypothesis, $(\mu^{f_{y}\to z}_{n})^{(t-1)}= (\mu^{z\to f_{\bmz}}_{n})^{(t-1)}$. Therefore, $\xi^{f_{\bmz}\to z}_{n}$ is updated by
    \beq
    \begin{split}
        (\xi^{f_{\bmz}\to z}_{n})^{(t)}+(\xi^{z\to f_{\bmz}}_{n})^{(t-1)}=\frac{1}{\bma_{n}^T(\bmC^{\bmxh;f_{\bmz}})^{(t-1)}\bma_{n}}\\
        =(\xi^{f_{\bmz}\to z}_{n})^{(t)}+(\mu^{f_{y}\to z}_{n})^{(t-1)}>0,
    \end{split}
    \label{eq:spawc2665}
    \eeq
    where we denote the covariance matrix at belief $b^{f_{\bmz}}$ updated just before \eqref{eq:spawc2664} as $(\bmC^{\bmxh;f_{\bmz}})^{(t-1)}$, and the one after the chain \eqref{eq:spawc2664} as $(\bmC^{\bmxh;f_{\bmz}})^{(t)}$. The second line of \eqref{eq:spawc2665} is a result of induction hypothesis and Lemma \ref{lemma:spawc262}.

    Since \eqref{eq:spawc2665} is greater than zero, $(\xi^{f_{\bmz}\to z}_{n})^{(t)}$ meets the threshold automatically. 
    
    Based on the update rule for $\mu^{f_{y}\to z}_{n}$, 
    \beq
        (\xi^{f_{y}\to z}_{n})^{(t)}+(\xi^{f_{\bmz}\to z}_{n})^{(t)}>0.
        \label{eq:spawc2656}
    \eeq
    Now we investigate the constraint for $(\xi^{z\to f_{\bmz}}_{n})^{(t)}$. We look at $(\bmC^{\bmxh;f_{\bmz}})^{(t)}$ just after the update chain $\eqref{eq:spawc2664}$:
    \beq
    \begin{split}
        (\bmC^{\bmxh;f_{\bmz}})^{(t)}\\
        =([(\bmC^{\bmxh;f_{\bmz}})^{(t-1)}]^{-1}+\bma_{n}[(\xi^{z\to f_{\bmz}}_{n})^{(t)}-(\xi^{z\to f_{\bmz}}_{n})^{(t-1)}]\bma_{n}^T)^{-1}.
    \end{split}
    \label{eq:spawc2667}
    \eeq
    The covariance matrix \eqref{eq:spawc2667} need to remain positive definite. By Lemma \ref{lemma:spawc262}, positive definiteness of \eqref{eq:spawc2667} is equivalent to the following constraint for $(\xi^{z\to f_{\bmz}}_{n})^{(t)}$
    \beq
        \begin{split}
            (\xi^{z\to f_{\bmz}}_{n})^{(t)}>-\frac{1}{\bma_{n}^T(\bmC^{\bmxh;f_{\bmz}})^{(t-1)}\bma_{n}}+(\xi^{z\to f_{\bmz}}_{n})^{(t-1)}\\
            \Leftrightarrow (\xi^{z\to f_{\bmz}}_{n})^{(t)}+(\xi^{f_{\bmz}\to z}_{n})^{(t)}>0,
        \end{split}
        \label{eq:spawc2668}
    \eeq
    where the second line follows from \eqref{eq:spawc2665}. Compare the update of $\mu^{z\to f_{\bmz}}_{n}$, \eqref{eq:spawc2656} with condition \eqref{eq:spawc2668}, and we find out that in ACEP, we can directly set $\mu^{z\to f_{\bmz}}_{n}=\mu^{f_{y}\to z}_{n}$ without checking the threshold.
\end{proof}
\begin{remark}
    The previous proposition indicates that in EPC, the belief $b^{f_{\bmz}}$ is always integrable. 
\end{remark}

Therefore, if we have the update order \eqref{eq:spawc2664}, the update for $\mu^{z\to f_{\bmz}}_{n}$ in both ACEP and EPC is:
\beq
    \mu^{z\to f_{\bmz}}_{n}(z_n)=\mu^{f_{y}\to z}_{n}(z_n)
\eeq
Otherwise, if we have different order, we apply the constrained KLD optimization in Section \ref{sect:constrainedKLDopt} based on \eqref{eq:spawc2660} and \eqref{eq:spawc2662p}.

\subsection{$x\to f_{x}$}
The belief at $x_k$ is
\beq
    b^{\xh}_{k}(x_{k})\propto \mu^{f_{x}\to x}_{k}(x_k)\mu^{f_{\bmz}\to x}_{k}.
\eeq
The belief is always integrable. The mean and variance of $b^{\xh}_{k}$ are
\beq
\begin{split}
    \tau^{\xh}_{k}=(\xi^{f_{x}\to x}_{k}+\xi^{f_{\bmz}\to x}_{k})^{-1}\\
    m^{\xh}_{k}=\tau^{\xh}_{k}(\nu^{f_{x}\to x}_{k}+\nu^{f_{\bmz}\to x}_{k}).
\end{split}
\eeq
To derive the threshold $\gamma^{x\to f_{x}}_{k}$, we observe the belief at $f_{x_k}$. If GMM prior is used, the sum of $\gamma^{x\to f_{x}}_{k}$ and the minimum component precision should be greater than zero.

\subsection{$x\to f_{\bmz}$}
To determine the threshold $\gamma^{x\to f_{\bmz}}_{k}$, we observe the rank one update. If the precision of message $\mu^{x\to f_{\bmz}}_{k}$ is updated from $\xi^{x\to f_{\bmz}}_{k}$ to $\gamma^{x\to f_{\bmz}}_{k}$, the new covariance matrix of belief $b^{f_{\bmz}}$ is updated by
\beq
    [(\bmC^{\bmxh;f_{\bmz}})^{-1}+\bme_{k}(\gamma^{x\to f_{\bmz}}_{k}-\xi^{x\to f_{\bmz}}_{k})\bme_{k}^T]^{-1},
    \label{eq:spawc2671}
\eeq
where $\bme_{k}$ denotes a unit vector with the $k$-th entry equal to one. Based on relation \eqref{eq:spawc2648p} and Lemma \ref{lemma:spawc262}, \eqref{eq:spawc2671} is positive definite iff. 
\beq
    \gamma^{x\to f_{\bmz}}_{k}>\xi^{x\to f_{\bmz}}_{k}-\xi^{\xh;f_{\bmz}}_{k}.
\eeq
\begin{proposition}
    In ACEP, with proper initialization (e.g. all positive message precision), if the update of $\mu^{f_{\bmz}\to x}_{k}$, $\mu^{f_{x}\to x}_{k}$, $\mu^{x\to f_{x}}_{k}$ and $\mu^{x\to f_{\bmz}}_{k}$ follows the following order:
    \beq
        \mu^{f_{\bmz}\to x}_{k} \rightarrow \mu^{x\to f_{x}}_{k} \rightarrow \mu^{f_{x}\to x}_{k}\rightarrow \mu^{x\to f_{\bmz}}_{k},
        \label{eq:spawc2673}
    \eeq
    the update of $\mu^{f_{\bmz}\to x}_{k}$ and $\mu^{x\to f_{\bmz}}_{k}$ does not need to evaluate the threshold.
\end{proposition}
\begin{proof}
    The proof is analog to the one of Proposition \ref{prop:spawc261}.
\end{proof}

Therefore, if we have the update order \eqref{eq:spawc2673}, the update
\beq
    \mu^{x\to f_{\bmz}}_{k}(x_{k})=\mu^{f_{x}\to x}_{k}(x_{k})
\eeq
automatically guaranties that $b^{f_{\bmz}}$ is integrable.

\section{Useful Relations}

An important relation based on Gaussian reproduction lemma is
\beq
\begin{split}
    \cbN(\bmH\bmx|\bma, \bmA)\cbN(\bmx|\bmb, \bmB)\\
    =\cbN(\bmx|\bmc, \bmC)\cbN(\bma|\bmH\bmb, \bmH\bmB\bmH^T+\bmA),
\end{split}
\eeq
where
\beq
\begin{split}
    \bmC=(\bmH^T\bmA^{-1}\bmH+\bmB^{-1})^{-1}\\
    \bmc=\bmC(\bmH^T\bmA^{-1}\bma+\bmB^{-1}\bmb).
\end{split}
\eeq

\begin{lemma} \label{lemma:spawc261}
    Let $\bma\in \mathbb{R}^{K}$, $\bmm_{\bmx}\in \mathbb{R}^{K}$, and $\bmC_{\bmx}\in\mathbb{R}^{K\times K} $ be symmetric. If 
    \beq
        \int \int \cbN(\bmx|\bmm_\bmx, \bmC_{\bmx})\cbN(z|m_{z}, \tau_z)\delta(z-\bma^T\bmx) dz d\bmx
        \label{eq:spawc26642}
        \eeq
    is integrable, then

    \beq
    \begin{split}
        \int \cbN(\bmx|\bmm_\bmx, \bmC_{\bmx})\cbN(z|m_{z}, \tau_z)\delta(z-\bma^T\bmx) d\bmx\\
            =\cbN(z| m_{\zh}, \tau_{\zh})\sqrt{(2\pi)^{K-1} (\bma^T\bmC_{\bmx}\bma)^{-1}\det(\bmC_{\bmx})}\\
            \cdot\cbN(0|m_{z}-\bma^T\bmm_{\bmx}, \bma^T\bmC_{\bmx}\bma+\tau_{z}),
    \end{split}
    \label{eq:spawc2642}
    \eeq
    where
    \beq
        \begin{split}
            \tau_{\zh}=[(\bma^T\bmC_{\bmx}\bma)^{-1}+\tau_{z}^{-1}]^{-1}\\
            m_{\zh}=\tau_{\zh}[(\bma^T\bmC_{\bmx}\bma)^{-1}\bma^T\bmm_{\bmx}+\tau_{z}^{-1}m_{z}]
        \end{split}
    \eeq
\end{lemma}
\begin{proof}
    We compute the integration by substitution. Define 
    \beq
        \bmU=\begin{bmatrix}
            \bmu_{1} &\dots &\bmu_{K-1}
        \end{bmatrix}\in \mathbb{R}^{K\times (K-1)},
    \eeq
    such that $\forall k,k'\in [1, K-1]: \bma^T\bmu_k=0$ and $\bmu_k^T\bmu_k'=\delta_{kk'}$, where $\delta_{kk'}$ denotes Kronecker delta that evaluates to zero if $k'\neq k$ and evaluates to one if $k=k'$.

    We construct a full rank basis matrix 
    \beq
        \bmH=\begin{bmatrix}
            \bma &\bmU
        \end{bmatrix}\in \mathbb{R}^{K\times K}.
    \eeq

    We substitute $\bmx$ by $\bmtheta$, where $\bmx=\bmH\bmtheta$ and thus, $d\bmx=|\det(\bmH)|d\bmtheta$.

    The integral in \eqref{eq:spawc2642} can be computed as
    \beq
        \begin{split}
            \int \cbN(\bmx|\bmm_\bmx, \bmC_{\bmx})\cbN(z|m_{z}, \tau_z)\delta(z-\bma^T\bmx) d\bmx\\
            =\int \int \cbN(\bma\theta_1+\bmU\bmthetab|\bmm_\bmx, \bmC_{\bmx})d\bmthetab\\
            \cdot\cbN(z|m_{z}, \tau_z)\delta(z-|\bma|^2\theta_1) |\det(\bmH)|d\theta_1\\
            =\int\int \cbN(\bmthetab|\bmm_{\bmthetab}(\theta_1), \bmC_{\bmthetab})d\bmthetab \exp[\frac{1}{2}\bmm_{\bmthetab}^T(\theta_1)\bmC_{\bmthetab}^{-1}\bmm_{\bmthetab}(\theta_1)]\\
            \cdot\exp[-\frac{1}{2}(\bmm_{\bmx}-\bma\theta_1)^T\bmC_{\bmx}^{-1}(\bmm_{\bmx}-\bma\theta_1)]\\
            \cbN(z|m_{z}, \tau_z)\delta(z-|\bma|^2\theta_1)
            \cdot|\det(\bmH)|d\theta_1\\
            =\sqrt{(2\pi)^{K-1}\det(\bmC_{\bmthetab})}\int \exp[\frac{1}{2}\bmm_{\bmthetab}^T(\theta_1)\bmC_{\bmthetab}^{-1}\bmm_{\bmthetab}(\theta_1)]\\
            \cdot\exp[-\frac{1}{2}(\bmm_{\bmx}-\bma\theta_1)^T\bmC_{\bmx}^{-1}(\bmm_{\bmx}-\bma\theta_1)]\\
            \cbN(z|m_{z}, \tau_z)\delta(z-|\bma|^2\theta_1)
            \cdot|\det(\bmH)|d\theta_1
        \end{split}
        \label{eq:spawc2646}
    \eeq
    where
    \beq
    \begin{split}
        \bmC_{\bmthetab}=(\bmU^T\bmC_{\bmx}^{-1}\bmU)^{-1}\\
        \bmm_{\bmthetab}(\theta_1)=\bmC_{\bmthetab}[\bmU^T\bmC_{\bmx}^{-1}(\bmm_{\bmx}-\bma\theta_1)]
        \end{split}
    \eeq
    The integral of \eqref{eq:spawc26642} exists, which implies that we have positive definite matrix
    \beq
        \bmC_{\bmthetab}\succ 0.
    \eeq
    From this point, we only have scalar variables $\theta_1$ and $z$.

    Continuing from \eqref{eq:spawc2646}, we have
    \begin{align}
        \int \cbN(\bmx|\bmm_\bmx, \bmC_{\bmx})\cbN(z|m_{z}, \tau_z)\delta(z-\bma^T\bmx) d\bmx \nonumber\\
        =\frac{|\det(\bmH)|\sqrt{(2\pi)^{K-1}\det(\bmC_{\bmthetab})}}{|\bma|^2}\label{eq:spawc26determinant}\\
        \cdot\exp[-\frac{1}{2}z^T(\bmb_{3}^T\bmC_{\bmx}^{-1}\bmb_{3}-\bmb_{1}^T\bmC_{\bmthetab}^{-1}\bmb_{1}+\tau_{z}^{-1})z] \label{eq:spawc2649}\\
        \cdot \exp[z^T(\bmb_{3}^T\bmC_{\bmx}^{-1}\bmm_{\bmx}-\bmb_{1}^T\bmC_{\bmthetab}^{-1}\bmb_{2}+\tau_{z}^{-1}m_{z})]\label{eq:spawc2650}\\
        \cdot \exp[-\frac{1}{2}(\bmm_{\bmx}^T\bmC_{\bmx}^{-1}\bmm_{\bmx}-\bmb_{2}^T\bmC_{\bmthetab}^{-1}\bmb_{2}+m_{z}^T\tau_{z}^{-1}m_{z})], \label{eq:spawc2651}
    \end{align}
    where
    \beq
        \begin{split}
            \bmb_{1}=\bmC_{\bmthetab}\bmU^T\bmC_{\bmx}^{-1}\frac{\bma}{|\bma|^2}\\
            \bmb_{2}=\bmC_{\bmthetab}\bmU^T\bmC_{\bmx}^{-1}\bmm_{\bmx}\\
            \bmb_{3}=\frac{\bma}{|\bma|^2};
        \end{split}
    \eeq
    We try to simplify \eqref{eq:spawc2649}. Let 
    \beq
    \begin{split}
        \bmQ=\begin{bmatrix}
            \bma^T/|\bma|\\
            \bmU^T
        \end{bmatrix}\bmC_{\bmx}^{-1}\begin{bmatrix}
            \bma/|\bma| &\bmU
        \end{bmatrix}
    \end{split}
    \label{eq:spawc2654}
    \eeq
    be a matrix of 2 blocks by 2 blocks. Since $\begin{bmatrix}
            \bma/|\bma| &\bmU
        \end{bmatrix}$ is a unitary matrix, we have
    \beq
        \bmQ^{-1}=\begin{bmatrix}
            \bma^T/|\bma|\\
            \bmU^T
        \end{bmatrix}\bmC_{\bmx}\begin{bmatrix}
            \bma/|\bma| &\bmU
        \end{bmatrix}
    \eeq
    According to Schur complement and matrix inversion lemma, the first two terms of the quadratic coefficient in \eqref{eq:spawc2649} can be verified to be
    \beq
    \begin{split}
        \bmb_{3}^T\bmC_{\bmx}^{-1}\bmb_{3}-\bmb_{1}^T\bmC_{\bmthetab}^{-1}\bmb_{1}=\frac{1}{|\bma|^2}([\bmQ^{-1}]_{11})^{-1}=(\bma^{T}\bmC_{\bmx}\bma)^{-1}
    \end{split}
    \label{eq:spawc2655}
    \eeq
    Now we try to simplify \eqref{eq:spawc2650}. The first two constant terms can be expanded as
    \beq
        \begin{split}
            \bmb_{3}^T\bmC_{\bmx}^{-1}\bmm_{\bmx}-\bmb_{1}^T\bmC_{\bmthetab}^{-1}\bmb_{2}=\bmb_{4}^T\bmm_{\bmx}.
        \end{split}
    \eeq
    where
    \beq
        \begin{split}
            \bmb_{4}^T=\frac{\bma^T}{|\bma|^2}\bmC_{\bmx}^{-1}-\frac{\bma^T}{|\bma|^2}\bmC_{\bmx}^{-1}\bmU(\bmU^T\bmC_{\bmx}^{-1}\bmU)^{-1}\bmU^T\bmC_{\bmx}^{-1}.
        \end{split}
        \label{eq:spawc2657}
    \eeq
    It has the following two properties
    \begin{itemize}
        \item $\bmb_{4}^T\frac{\bma}{|\bma|^2}$ is the same as expanding the left hand side of \eqref{eq:spawc2655}. Therefore,
        \beq
            \bmb_{4}^T\frac{\bma}{|\bma|^2}=(\bma^{T}\bmC_{\bmx}\bma)^{-1}.
            \label{eq:spawc2658}
        \eeq

        \item It is in the null space of $\bmU$:
        \beq
            \bmb_{4}^T\bmU=0
            \label{eq:spawc2659}
        \eeq
    \end{itemize}

    Therefore, from \eqref{eq:spawc2659}, we know $\bmb_{4}$ must be a scalar multiple of $\bma$. By \eqref{eq:spawc2658}, we have
    \beq
        \bmb_{4}^T=(\bma^{T}\bmC_{\bmx}\bma)^{-1}\bma^T.
    \eeq

    We investigate \eqref{eq:spawc2651}, the first two terms can be expanded as
    \beq
        \bmm_{\bmx}^T\bmC_{\bmx}^{-1}\bmm_{\bmx}-\bmb_{2}^T\bmC_{\bmthetab}^{-1}\bmb_{2}=\bmb_{5}^T\bmm_{\bmx},
    \eeq
    where
    \beq
        \bmb_{5}^T=\bmm_{\bmx}^T\bmC_{\bmx}^{-1}-\bmm_{\bmx}^T\bmC_{\bmx}^{-1}\bmU(\bmU^T\bmC_{\bmx}^{-1}\bmU)^{-1}\bmU^T\bmC_{\bmx}^{-1}
        \label{eq:spawc2662}
    \eeq

    Similar as before, $\bmb_{5}^T$ has the following two properties
    \begin{itemize}
        \item Compare \eqref{eq:spawc2657} with \eqref{eq:spawc2662}, we can verify that
        \beq
            \bmb_{5}^T\frac{\bma}{|\bma|^2}=(\bmb_{4}^T\bmm_{\bmx})^T=\bmm_{\bmx}^T\bma(\bma^{T}\bmC_{\bmx}\bma)^{-1}
            \label{eq:spawc2663}
        \eeq
        \item It is also in the null space of $\bmU$
        \beq
            \bmb_{5}^T\bmU=0
        \eeq
    \end{itemize}
    Therefore, $\bmb_{5}^T$ must be a scalar multiple of $\bma$. From \eqref{eq:spawc2663}, we have
    \beq
        \bmb_{5}^T=\bmm_{\bmx}^T\bma(\bma^{T}\bmC_{\bmx}\bma)^{-1}\bma^T
    \eeq

    Finally, we examine the determinant in \eqref{eq:spawc26determinant}. 
    Since 
    \beq
        \bmH=\begin{bmatrix}
            \bma/|\bma| &\bmU
        \end{bmatrix}\begin{bmatrix}
            |\bma| &\\
            &\bmI
        \end{bmatrix},
    \eeq
    we have $|\det(\bmH)|=|\bma|$. 

    From \eqref{eq:spawc2654}, we have
    \beq
        \bmC_{\bmthetab}=([\bmQ]_{22})^{-1},
    \eeq
    where $[\bmQ]_{22}$ denote the matrix at the second block row and the second block column of $\bmQ$.
    Based on block determinant property and \eqref{eq:spawc2655}, we have
    \beq
        \det(\bmC_{\bmthetab})=\det(\bmC_{\bmx})(\bma^T\bmC_{\bmx}\bma)^{-1}|\bma|^2
    \eeq

    Substitute the discussion results into \eqref{eq:spawc26determinant} - \eqref{eq:spawc2651} and we have
    \beq
        \begin{split}
            \int \cbN(\bmx|\bmm_\bmx, \bmC_{\bmx})\cbN(z|m_{z}, \tau_z)\delta(z-\bma^T\bmx) d\bmx\\
            =\sqrt{(2\pi)^{K-1} (\bma^T\bmC_{\bmx}\bma)^{-1}\det(\bmC_{\bmx})}\\
            \cdot\cbN(0|m_{z}-\bma^T\bmm_{\bmx}, \bma^T\bmC_{\bmx}\bma+\tau_{z})\\
            \cdot\cbN(z| m_{\zh}, \tau_{\zh}),
        \end{split}
    \eeq
    where 
    \beq
        \begin{split}
            \tau_{\zh}=[(\bma^T\bmC_{\bmx}\bma)^{-1}+\tau_{z}^{-1}]^{-1}\\
            m_{\zh}=\tau_{\zh}[(\bma^T\bmC_{\bmx}\bma)^{-1}\bma^T\bmm_{\bmx}+\tau_{z}^{-1}m_{z}]
        \end{split}
    \eeq

\end{proof}

\begin{lemma}
\label{lemma:spawc262}
    Let $\bmC^{-1}\in \mathbb{R}^{K\times K}$ be a symmetric positive definite matrix, then $(\bmC^{-1}+\bma\tau^{-1}\bma^T)^{-1}$ is also positive definite iff. $\tau^{-1}>-\frac{1}{\bma^T\bmC\bma}$.
\end{lemma}

\begin{proof}
    The matrix inversion does not change the positive definiteness. Therefore, we investigate $\bmC^{-1}+\bma\tau^{-1}\bma^T$. Since $\bmC^{-1}$ is positive definite and symmetric, we factorize it by
    \beq
        \bmC^{-1}=\bmC^{-1/2}\bmC^{-T/2}, 
    \eeq
    such that $\bmC^{-1/2}$ is also a symmetric positive definite matrix.
    Therefore, we have
    \beq
    \begin{split}
        \bmC^{-1}+\bma\tau^{-1}\bma^T=\bmC^{-1/2}(\bmI+\bmb\tau^{-1}\bmb^T)\bmC^{-T/2}
    \end{split}
    \label{eq:spawc2691}
    \eeq
    where
    \beq
        \bmb=\bmC^{1/2}\bma.
    \eeq
    Since $\bmC^{-1/2}$ is full-rank, \eqref{eq:spawc2691} is positive definite iff. $(\bmI+\bmb\tau^{-1}\bmb^T)$ is positive definite.
    We construct a unitary matrix $\bmU=\begin{bmatrix}\bmu_1 &\dots &\bmu_{K-1} \end{bmatrix}\in \mathbb{R}^{K\times(K-1)}$, such that $\forall k,k'\in [1, K-1]: \bmb^T\bmu_k=0$ and $\bmu_k^T\bmu_k'=\delta_{kk'}$, where $\delta_{kk'}$ denotes Kronecker delta that evaluates to zero if $k'\neq k$ and evaluates to one if $k=k'$.

    Another square unitary matrix can be constructed by
    \beq
        \bmH=\begin{bmatrix}
            \frac{\bmb}{|\bmb|} &\bmU
        \end{bmatrix}\in \mathbb{R}^{K\times K}.
    \eeq
    Therefore,
    \beq
    \begin{split}
        \bmI+\bmb\tau^{-1}\bmb^T=\bmH\begin{bmatrix}
            1+|\bmb|^2\tau^{-1} &\\
            &\bmI
        \end{bmatrix}\bmH^T.
    \end{split}
    \label{eq:spawc2694}
    \eeq
    From \eqref{eq:spawc2694}, $(\bmC^{-1}+\bma\tau^{-1}\bma^T)^{-1}$ is positive definite iff.
    \beq
    \begin{split}
        1+|\bmb|^2\tau^{-1}>0\Leftrightarrow\tau^{-1}>-\frac{1}{\bma^T\bmC\bma}
    \end{split}
    \eeq
\end{proof}
\begin{remark}
    Lemma \ref{lemma:spawc262} implies that it is easier to use precision instead of variance when deriving. When using the precision representation, the feasible set is $\tau^{-1}\in (-\frac{1}{\bma^T\bmC\bma}, +\infty)$. However, in variance representation, this condition is equivalent to 
    $\tau\in (-\infty, \bma^T\bmC\bma)\cup(0, +\infty)$.
\end{remark}

\section{Gaussian Mixture Model (GMM)}
For simulations, we use GMM prior $p(x_k)$ and likelihood $p(y_n|z_n)$.

Consider the following belief
\beq
    b(\theta)\propto f(\theta)\mu(\theta),
\eeq
where 
\beq
    f(\theta)=\sum_{s}w_{s}\cbN(\theta|m_{s}, \tau_{s})=\sum_{s}w_{s}\cbM(\theta|\nu_{s}, \xi_{s})
\eeq
and
\beq
    \mu(\theta)=\cbN(\theta|m, \tau)=\cbM(\theta| \nu, \xi).
\eeq
Therefore, 
\beq
    b(\theta)=\sum_{s}w_s\cbN(0|m_s-m, \tau_s+\tau) \cbM(\theta|\nu_s+\nu, \xi_s+\xi).
\eeq
Since $b(\theta)$ is integrable, we have $\forall s: \xi_s+\xi>0$. We can transform the factor containing $\theta$ to normalized Gaussian
\beq
    \cbM(\theta|\nu_s+\nu, \xi_s+\xi)=\sqrt{2\pi \tau_{\thetah|s}}\mathcal{N}(\theta|m_{\thetah|s}, \tau_{\thetah|s}),
\eeq
where 
\beq
    \begin{split}
        \tau_{\thetah|s}=\frac{1}{\xi_s+\xi}\\
        m_{\thetah|s}=\frac{\nu_{s}+\nu}{\xi_{s}+\xi}.
    \end{split}
\eeq
We rewrite the belief as
\beq
    b(\theta)\propto \sum_{s} \eta_{s} \mathcal{N}(\theta|m_{\thetah|s}, \tau_{\thetah|s}),
\eeq
where 
\beq
    \eta_{s}=w_s\cbN(0|m_s-m, \tau_s+\tau)\sqrt{2\pi \tau_{\thetah|s}}.
\eeq
We can normalize the belief. Define
\beq
    \rho_s=\frac{\eta_s}{\sum_{s'}\eta_{s'}},
\eeq
and we have
\beq
    b(\theta)=\sum_s \rho_s \mathcal{N}(\theta|m_{\thetah|s}, \tau_{\thetah|s}).
\eeq
The belief mean and variance are
\beq
    \begin{split}
        &m_{\thetah}=\E_{b}[\theta]=\sum_s \rho_s m_{\thetah|s}\\
        &\tau_{\thetah}=\var_{b}[\theta]=\left[\sum_s \rho_s(\tau_{\thetah|s}+|m_{\thetah|s}|^2)\right]-|m_{\thetah}|^2.
    \end{split}
\eeq

\section{Simulation Results}
We consider a system of size $N\times K=8\times 12$. The prior distribution of each symbol is independent and follows GMM with exponential decay:
\beq
\begin{split}
    p(x_{k})=\frac{1}{2}\mathcal{N}(x_k|2^{-n+1}, 0.1\times 2^{-2(n-1)})\\
    +\frac{1}{2}\mathcal{N}(x_{k}|-2^{-n+1}, 0.1\times 2^{-2(n-1)}).
\end{split}
\eeq
Each entry in the measurement matrix $\bmA$ is generated from i.i.d. Gaussian. Additive white Gaussian noise is assumed, whose power is adjusted such that the SNR is 15 dB.  We generate 1000 independent realizations.

\begin{figure}[t]
    \centering
    \includegraphics[width=0.48\textwidth]{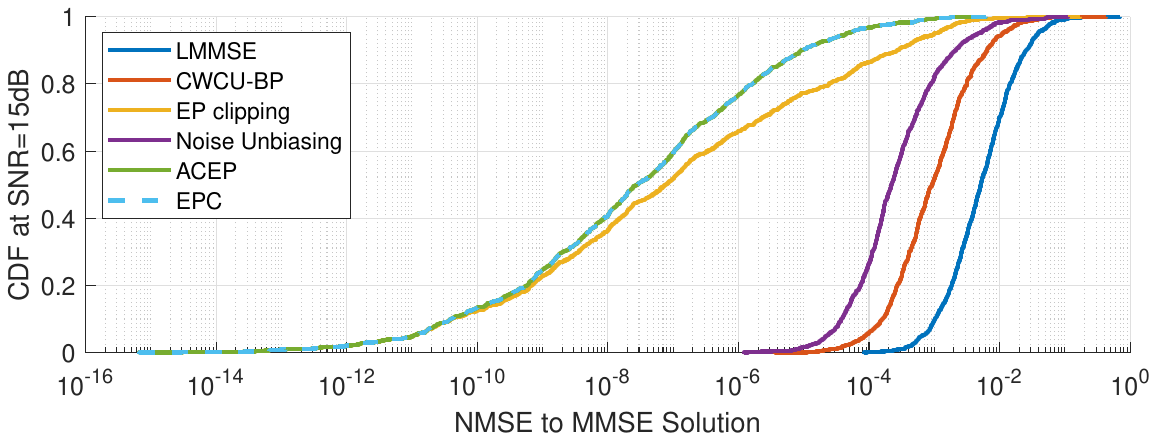}
    \vspace{-2mm}
    \caption{CDF of NMSE} 
    \vspace{-2mm}
    \label{fig:results_var}
    \vspace{-4mm}
\end{figure}

Besides our proposed methods, we also simulated the results for LMMSE, CWCU-BP \cite{huemer2014cwcu, Triki:asilo05}, EP clipping \cite{ngo2020multi}, and Noise Unbiasing \cite{fischer2020vamp}. The simulation results show that the performance of our proposed methods surpasses the existing attempts to handle the non-integrable beliefs. This indicates that in order to have more accurate estimates, we have to accept the non-integrable messages.

\section{Conclusions}

We proposed two EP frameworks that avoid non-integrable beliefs without shrinking the feasible set. Both methods permits non-integrable messages and factors. The proposed methods are then applied to GLM. Based on our analysis, we find out that if EPC is used, the only beliefs that require checking are $b^{f_{y}}_{n}$ and $b^{f_{x}}_{k}$. On the other hand, if ACEP is used, only the messages to and from factor nodes $f_{y_n}$ and $f_{x_k}$ need to be compared with the thresholds to guarantee integrable beliefs. The simulation results indicate that in order to achieve better results, we must accept non-integrable messages.

{\bf Acknowledgements}
EURECOM's research is partially supported by its industrial members:
ORANGE, BMW, SAP, iABG,  Norton LifeLock, by the French PEPR-5G projects PERSEUS and YACARI, the EU H2030 project CONVERGE, and by a Huawei France funded Chair towards Future Wireless Networks.



{
\bibliographystyle{IEEEtran}
\bibliography{asilomar23,SBL_ref}
}


\end{document}